\documentclass{esannV2}
\usepackage[dvips]{graphicx}
\usepackage{amssymb,amsmath,array}

\usepackage{url}
\usepackage{xcolor}         
\usepackage[numbers, sort, comma, square]{natbib}

\usepackage{graphicx}
\usepackage{floatrow}
\usepackage{subcaption}
\usepackage{amsmath}
\usepackage{amssymb}
\usepackage{amsmath}
\usepackage{amsthm}
\newtheorem{definition}{Definition}

\newtheorem{proposition}{Proposition}

\usepackage{bbm}
\usepackage{algorithm}
\usepackage{algpseudocode}
\newcommand{\A}[1]{\textbf{A#1}}
\renewcommand\Re{\mathbb{R}}

\newfloatcommand{capbtabbox}{table}[][\FBwidth]

\begin{document}
	\title{A Baseline for Shapley Values in MLPs: from Missingness to Neutrality}
	
	\author{Cosimo Izzo, Aldo Lipani, Ramin Okhrati, Francesca Medda
		%
		%
		\vspace{.3cm}\\
		%
		University College London, London, United Kingdom
	}
	
	\maketitle
	
	\begin{abstract}
		Deep neural networks have gained momentum based on their accuracy, but their interpretability is often criticised. As a result, they are labelled as black boxes. In response, several methods have been proposed in the literature to explain their predictions. Among the explanatory methods, Shapley values is a feature attribution method favoured for its robust theoretical foundation. However, the analysis of feature attributions using Shapley values requires choosing a baseline that represents the concept of missingness. An arbitrary choice of baseline could negatively impact the explanatory power of the method and possibly lead to incorrect interpretations. In this paper, we present a method for choosing a baseline according to a neutrality value: as a parameter selected by decision-makers, the point at which their choices are determined by the model predictions being either above or below it. Hence, the proposed baseline is set based on a parameter that depends on the actual use of the model. This procedure stands in contrast to how other baselines are set, i.e. without accounting for how the model is used. We empirically validate our choice of baseline in the context of binary classification tasks, using two datasets: a synthetic dataset and a dataset derived from the financial domain.
	\end{abstract}
	
	\section{Introduction and Background}
In disciplines such as economics, finance and healthcare, the ability to explain predictions is as important as having a model that performs well \cite{tsang2017detecting, goodman2017european}.
%
A solution to the problem is provided by feature attribution methods, which are used to indicate how much each feature contributes to the prediction for a given example.
A theoretically grounded feature attribution method 
is provided by Shapley values and their approximations \cite{lundberg2017unified}. 
When explaining a prediction with Shapley values, we need to perform two steps. First, we define a \textit{baseline}. Then, we compute the Shapley values for a given example.
While most works focused on the latter, the former has not been sufficiently explored, despite the implications that the baseline definition carries to correctly interpret Shapley values.


For a neural network $G_{\theta}$ with parameters $\theta$, and $n$ input features, the contribution of the feature $j$ calculated according to the Shapley value for the input $x=[x_1, x_2, \dots, x_n]$ 
is given by:
\begin{equation}
\sum_{S \subseteq P \setminus \{j\}} \frac{|S|!(|P|-|S|-1)!}{|P|!}\left(
G_{\theta}(\tilde{x}_{S \cup \{j\}}) - 
G_{\theta}(\tilde{x}_{S})\right),
\end{equation}
where $P$ is the collection of all feature indexes, the element $i$ of the vector $\tilde{x}_S$ is given by $\tilde{x}_{S,i}=x_i \mathbbm{1}_{\{i\in S\}} + b_i \mathbbm{1}_{\{i \not\in S\}}$ (similarly for $\tilde{x}_{S\cup \{j\}}$), and $b_i$ is the baseline value for the feature $i$. The \textit{baseline} models the missingness of a feature, i.e., it replaces that feature when it is absent. As it is argued by \citet{sturmfels2020visualizing}, the concept of missingness is not well defined and explored in machine learning. 
The standard practice in setting up the baseline is to assign a vector of \textbf{zeros} \cite{zeil2014,Sundararajan2017,shrikumar2017learning,Ancona2019} for all features, which coincides with the \textbf{average} vector baseline when features are standardised. However, this choice
could be misleading. For example, in a classification task, with binary features
representing the presence or absence of an entity, given an
example and its prediction value, such a baseline would always
measure a null contribution for each feature with value
equal to zero. A way to address this zero-baseline insensitivity problem 
is to use the maximum distance baseline (\textbf{mdb}) \cite{sturmfels2020visualizing}. 
This baseline consists in taking the furthest observation from the current one in an $L^1$ norm. This approach unequivocally creates incoherent justifications for the interpretations provided by the model due to the correlation of the baseline with the underlying dataset.
Alternatively, one can consider a sample of baselines and average the attributions computed over each baseline 
\cite{smilkov2017smoothgrad,lundberg2017unified,erion2019learning,sundararajan2019many}, for example by using the underlying empirical distributions of the dataset (\textbf{p}$_{X}$) \cite{lundberg2017unified,sundararajan2019many,sturmfels2020visualizing}. However, the \textbf{p}$_{X}$ baseline increases the computational cost of estimating feature attributions linearly with respect to the number of draws. Moreover, this choice of baseline does not allow the setting of a reference value on the model output when computing the Shapley values. This is important when decisions are taken with respect to a specific value of the model.

%


The evaluation of explainability methods from a quantitative perspective is difficult due to the lack of a clear definition of what is a correct explanation \cite{Sundararajan2017,Ancona2019}.
Many extrinsic evaluation of explainability power have been developed. 
The intuition behind these methods is that if the feature attribution method correctly identifies the most important feature, then when this feature is removed, the model performance should decrease more (or the prediction value should deviate more) than when a less important feature is removed \cite{hooker2019benchmark, Anconaetal2017}.
A limitation of these evaluation methods is that since a feature is removed once at the time, these measures may be mislead by potential feature correlations. In this paper, to avoid this issue, we also perform an analysis using a synthetic dataset where features are generated independently.
Further, we evaluate the explainability power across two dimensions. First, along the feature importance dimension by means of ROAR \cite{hooker2019benchmark}. ROAR consists of, given a model, to first identify the most important features on a per example basis, then retrain the model on a dataset where the these features are replaced with their average values, and measure the difference in performance between the two models.
Second, along the information gain dimension when removing features by means of \textit{absolute logits} ($|\log(G_{\theta}(x)/(1-G_{\theta}(x)))|$ where $G_{\theta}$ is a probabilistic classifier and $x$ an example). 
%
The choice of this measure is motivated also by information theory. Indeed, standard logits can be seen as the difference between two Shannon information. Since in a binary classification task the Shannon information differential represents the confidence of the model in classifying the instance as positive or as negative, we take the absolute value of the standard logits to measure the variation in Shannon information in both directions. As this measure decreases (increases), the Shannon information differential decreases (increases). Furthermore, when the Shannon information differential is 0 we can argue that the model becomes uninformative. Indeed, by doing so we are testing whether there is convergence to a meaningful value that represents missingness for the user of the model.
	\label{sec:introduction}
%
	
%
	\section{The Neutral Baseline}
	\label{sec:approach}
	
In this section, we theoretically justify the existence of a baseline according to well defined concepts of neutrality value and fair baselines. Following \citet{Bach2015}, we argue that the baseline should live on the decision boundary of the classifier.
\begin{definition}[Neutrality Value]
\label{def:neutrality_value}
Given a model prediction $\hat{y}$ and a decision maker, we say that the value $\alpha$ is neutral if the decision maker's choice is determined by the value of $\hat{y}$ being either below or above $\alpha$.
\end{definition}
Generally, the concept of missingness is domain specific. However, every time we are faced with a decision boundary and a model, a natural choice is to consider missingness as missing information from the model to the user to take a choice, i.e.: when the model is at the neutrality value. The idea is that this neutrality value can lead to a point in the input domain that could be used as a baseline. 
However, given a neutrality value and a single-layer perceptron (SLP) with more than one continuous input feature, there are an infinite number of possible combinations of such inputs that lead to the same neutral output. 
Nevertheless, it is possible to narrow down the set of candidates by being fair in representing each feature in its input space, and given its relation to model $G_{\theta}(.)$. This ensures 
that Shapley values are not biased by distributional differences. We formalise this by introducing the concept of \textit{fair baselines}:
%
%
\begin{definition}[Space of Fair Baselines]
	\label{def:space_baselines}
	Consider a dataset in $\Re^k$, $k\geq1$, generated by a distribution. The set of fair baselines for a monotonic model $G_{\theta}(.)$ 
	is given by:
	$\tilde{B}=\{x^p \in \Re^k : x^{p}_j = C^{-1}_{j}(\mathbbm{1}_{\theta_j>0}\cdot p+\mathbbm{1}_{\theta_j \leq 0}\cdot (1-p)), \, p \in [0,1], j=1,2,\dots,k\}$,
	where $C^{-1}_{j}$ is the inverse marginal CDF of $x_j \; \forall j$.  
\end{definition}

Based on the two definitions, neutrality value and space of fair baselines, in what follows we demonstrate the existence of a fair baseline that when given to a SLP returns the neutrality value. Before doing this, we need to state the following two assumptions: \A1. All activation functions are monotonic and continuous. \A2. All marginal cumulative distribution functions (CDFs) of the joint CDF of the input features are bijective and continuous.
Using Definitions \ref{def:neutrality_value} and \ref{def:space_baselines}, and Assumptions \A1 and \A2, the following proposition guarantees 
the existence of a neutral and fair baseline for SLPs:\footnote{It can be proved by contradiction that, if monotonicity in \A1 is replaced by strict monotonicity, the solution becomes unique.}
\begin{proposition}
\label{pro:baseline_for_slps}
Given an SLP ($G_{\theta}$) satisfying \A1, a dataset satisfying \A2, and a neutrality value $\alpha$ in the image of $G_{\theta}$, then there exists at least a fair baseline $x$ such that $G_{\theta}(x)=\alpha$.
\end{proposition}

\begin{proof}
	We need to prove that $\alpha\in G_{\theta}(\tilde{B})$ where $G_{\theta}(\tilde{B})$ is the image of $\tilde B$ under $G_\theta$. Suppose that $I$ is the image of the SLP. We show that $I \subseteq G_{\theta}(\tilde{B})$ which proves the result, since $\alpha \in I$.
	We start by showing that $\text{inf }G_{\theta}(\tilde{B}) \leq \text{inf }I$ and that $\text{sup }G_{\theta}(\tilde{B}) \geq \text{sup }I$. Consider vector $x^0 \in \tilde{B}$ defined by $x^0 = \{x^{0}_j = C^{-1}_{j}(\mathbbm{1}_{\theta_j\leq0}), \text{ for all } j=1,2,\dots,k\}$. So  elements of $x^0$ are the smallest possible if the coefficients are positive, and the largest possible when they are negative. From Assumption \textbf{A1}, it follows that $G_{\theta}(x^0)$ is the smallest value that the SLP can take. Hence, $G_{\theta}(x^0)\leq \text{inf }I$. 
	Let us now take the vector $x^1 \in \tilde{B}$ which is defined by: $x^1 = \{x^{1}_j = C^{-1}_{j}(\mathbbm{1}_{\theta_j>0}), \; \text{ for all } j=1,2,\dots,k\}$. So  elements of $x^1$ are the largest possible if the  coefficients are positive, and the smallest possible when they are negative. From assumption \textbf{A1}, it follows that $G_{\theta}(x^1)$ is the largest value that the SLP can take. Hence, $G_{\theta}(x^1) \geq \text{sup }I$. 
	Suppose that $\alpha$ is in the image of the SLP. Define function $h:[0,1]\rightarrow G(\tilde B)$ by $h(p)=G_\theta(x^p)=G(\theta\cdot (x^p)^\top)$ where $x^p\in \tilde B$, i.e, $x^{p}_j = C^{-1}_{j}(\mathbbm{1}_{\theta_j>0}\cdot p+\mathbbm{1}_{\theta_j\leq0}\cdot (1-p))$, $j=1,2,\dots,k$. From the above argument, we have that $h(0)=G_\theta(x^0)\leq \alpha \leq G_\theta(x^1)=h(1)$. Since $G$ and $C^{-1}$ are continuous functions by \textbf{A1} and \textbf{A2}, $h$ is also continuous. By the intermediate value theorem, there is a $p^*\in[0,1]$ such that $h(p^*)=\alpha$, which means that $G_\theta(x^{p^*})=\alpha$. 
\end{proof}

This proof suggests a way to find one neutral and fair baseline for a SLP. An algorithm using empirical CDFs instead of theoretical ones, 
requires as inputs an SLP ($G_{\theta}$), a neutrality value ($\alpha$), a quantile function 
for each dimension of input features, 
a granularity level $\delta > 0$, and a tolerance level $\epsilon>0$. $\delta$ and $\epsilon$ control the speed of search and the margin of error in finding a baseline such that $|G_{\theta}(x) - \alpha| < \epsilon$.
This algorithm starts the search from the lowest output value, which is when $p = 0$, 
and it stops when it reaches a point which is close enough to $\alpha$. 
This is possible because using the parameters of the SLP we can restrict and define an order in function of $p$ for the set of fair baselines. This allows us to test these baselines from the smallest to the largest SLP value.

Finding a baseline for MLPs is more complicated, because there is no easy way to order the baselines in function of $p$ as in the SLP case, unless the MLP is monotonic in each of the features.
Nevertheless, we observe that a MLP with $L$ layers can be rewritten as a function of $\sum_{l=2}^{L} \prod_{l'=l}^{L} k_{l'}$ SLPs. This is done by replicating every node at layer $l$, $k_{l+1}$ times, i.e., the number of nodes at layer $l+1$, 
and considering every node in the layer $l+1$ as a SLP with input given by the layer $l$.
Based on this observation, we can recursively apply the algorithm for the SLP backwards through the layers of the model to recover the neutrality values across those SLPs, from the output layer to the input layer. 
This will provide $\prod_{l=2}^L k_l$ baselines, one for each SLP in the first hidden layer. 
Finally, in order to aggregate these baselines, we define an equivalent sparse representation of a MLP (SparseMLP), which is constructed by concatenating each of the SLPs defined above. 
See Fig.~\ref{fig:SparseMLP} for an example.
This representation allows us to compute the Shapley value for each example-feature pair by using all fair baselines found at once.
\begin{figure*}[!h]
	\centering
	\includegraphics[width=0.9\linewidth,trim=100 60 60 20, clip]{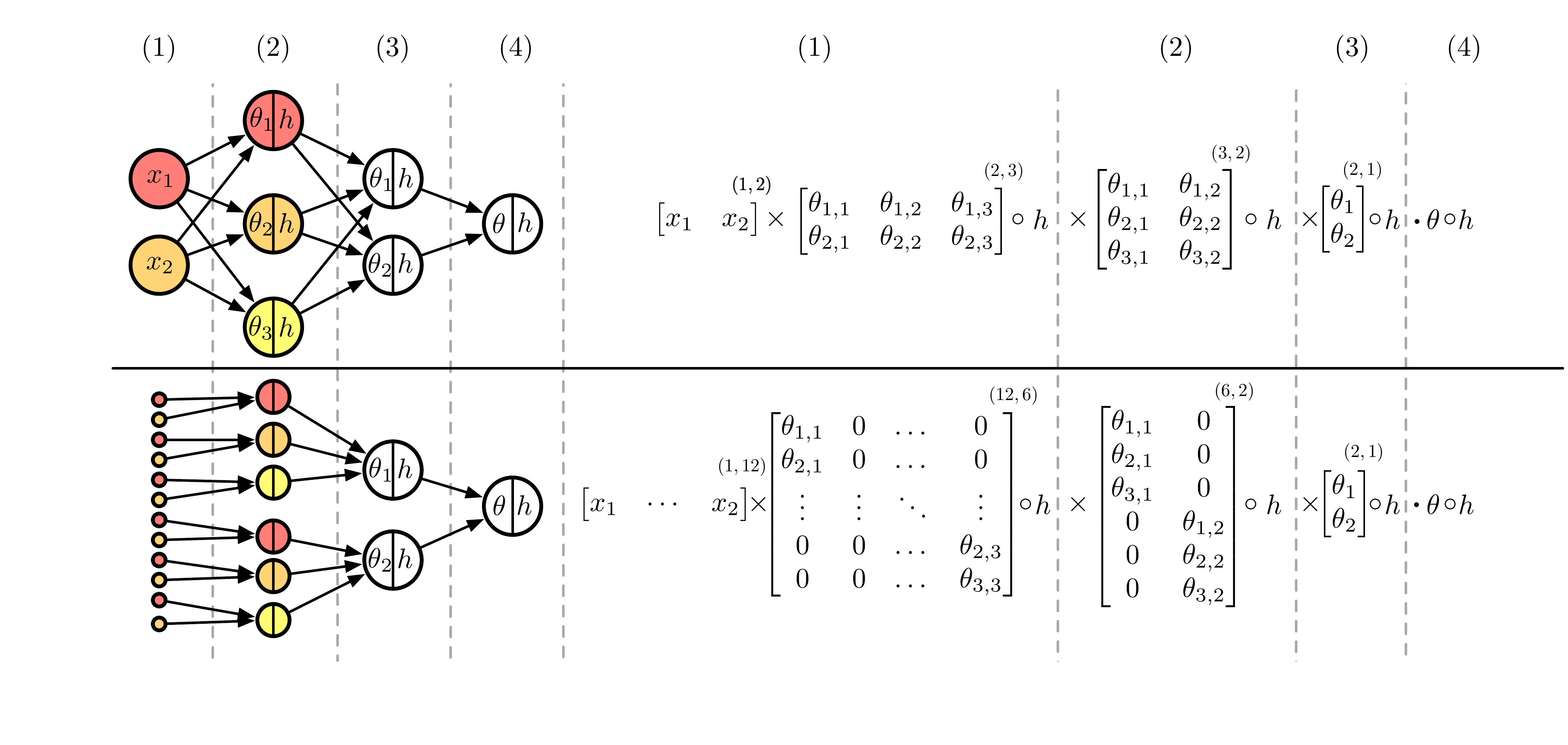}
	\caption{A MLP (above) and its equivalent sparse representation (below).}
	\label{fig:SparseMLP}
\end{figure*}

	\section{Experiments}
	\label{sec:experiments}
	We evaluate in classification tasks the explainability power of our baseline method (\textbf{neutral}$_\alpha$) against the
\textbf{zero}, 
\textbf{average}, 
\textbf{p}$_{X}$, and 
\textbf{mdb} baselines.
The code used to run these experiments is available at the following weblink: \url{https://github.com/cosimoizzo/Neutral-Baseline-For-Shapley-Values}.
Since the output of the trained classifier is probabilistic, i.e., its codomain is in $[0, 1]$ we set the neutrality value $\alpha$ (and so the decision boundary) to $0.5$.
We use two datasets, a synthetic and a real one. The former to simulate a dataset with independent features and controlled feature importance. The latter to experiment with a real use case.

For the synthetic dataset, we generate 5 independent features from a multivariate normal distribution with 0 mean and unit variance. The importance and sign of each feature are randomly drawn and the partition of the space in two classes is nonlinear. We repeat this 100 times, thus generating 100 synthetic datasets. 

The real dataset is about default of credit card clients \cite{yeh2009comparisons}. The dataset contains one binary target variable and 23 features. The number of observations is 29,351. In order to apply ROAR to such dataset, we need to reduce the number of observations to at least 300. We do so by sampling these observations while keeping the two classes balanced. Additionally, to further reduce the computational cost we use Shapley sampling \cite{castro2009polynomial,okhrati2020multilinear}.
%

We validate on both datasets a MLP with sigmoid activation functions and binary cross-entropy loss, and we use Adam as optimiser. The number of hidden layers and neurons in each layer are chosen via a Monte Carlo sampling of models using the training and validation sets.
%
\begin{figure}[!h]
	\centering
	\begin{subfigure}{1\textwidth}
		\centering
		\includegraphics[width=1.25\linewidth]{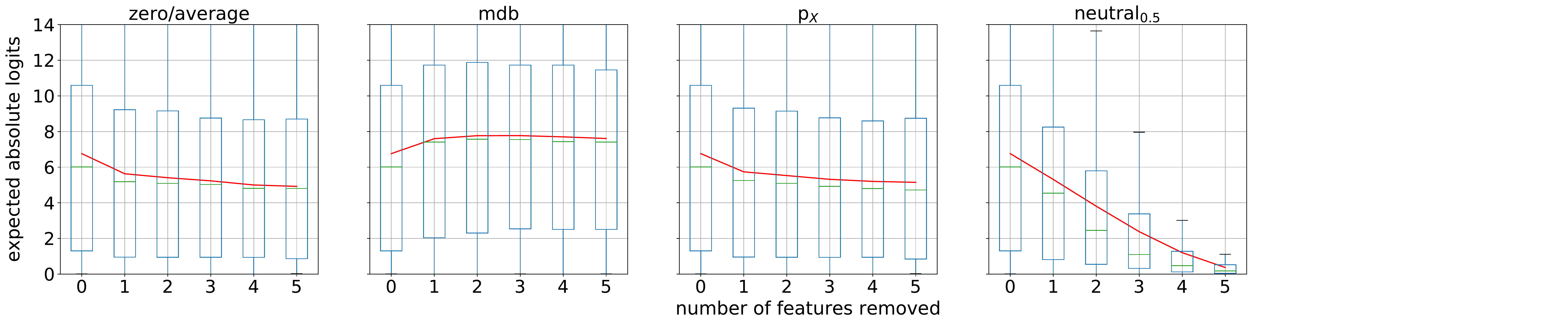}
		\caption{Information content synthetic.}
		\label{fig:local_synth}
	\end{subfigure}%
		\centering
	
	\begin{subfigure}{1\textwidth}
		\centering
		\includegraphics[width=0.55\linewidth]{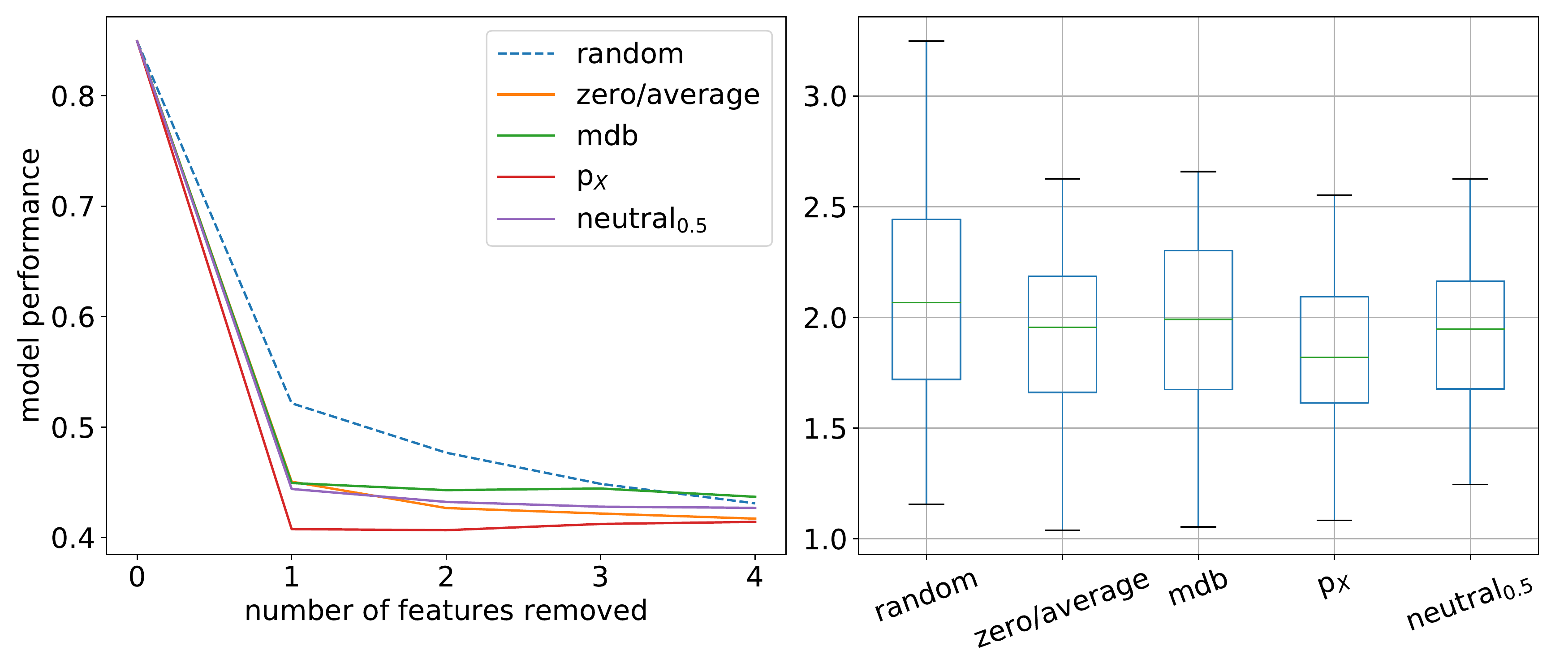}
		\caption{ROAR synthetic.}
		\label{fig:roar_synth}
	\end{subfigure}
	\caption{Perturbation tests on the synthetic data. 
		}
\end{figure}

\begin{figure}[!h]
	\centering
	\begin{subfigure}{1\textwidth}
		\centering
		\includegraphics[width=1.20\linewidth]{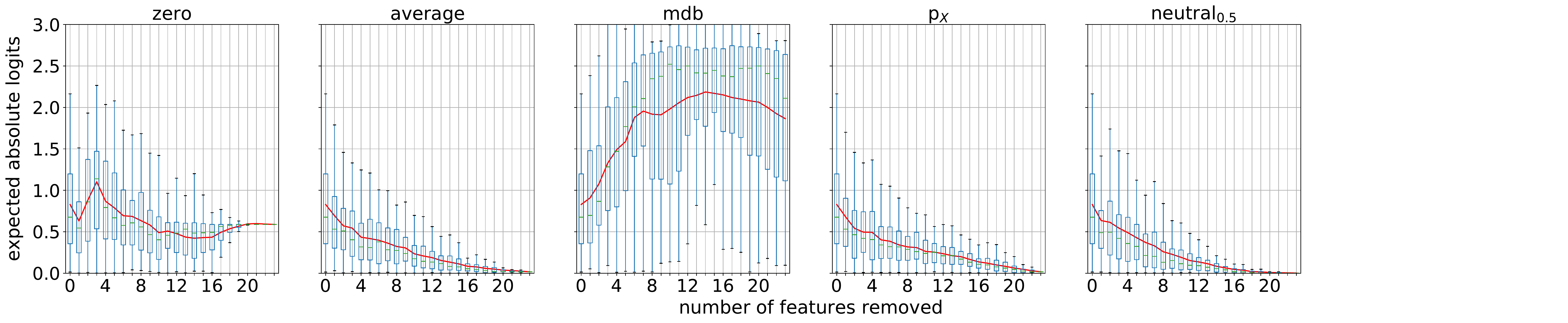}
		\caption{Information content credit card.}
		\label{fig:local_ccd}
	\end{subfigure}%

	\begin{subfigure}{1\textwidth}
		\centering
		\includegraphics[width=0.6\linewidth]{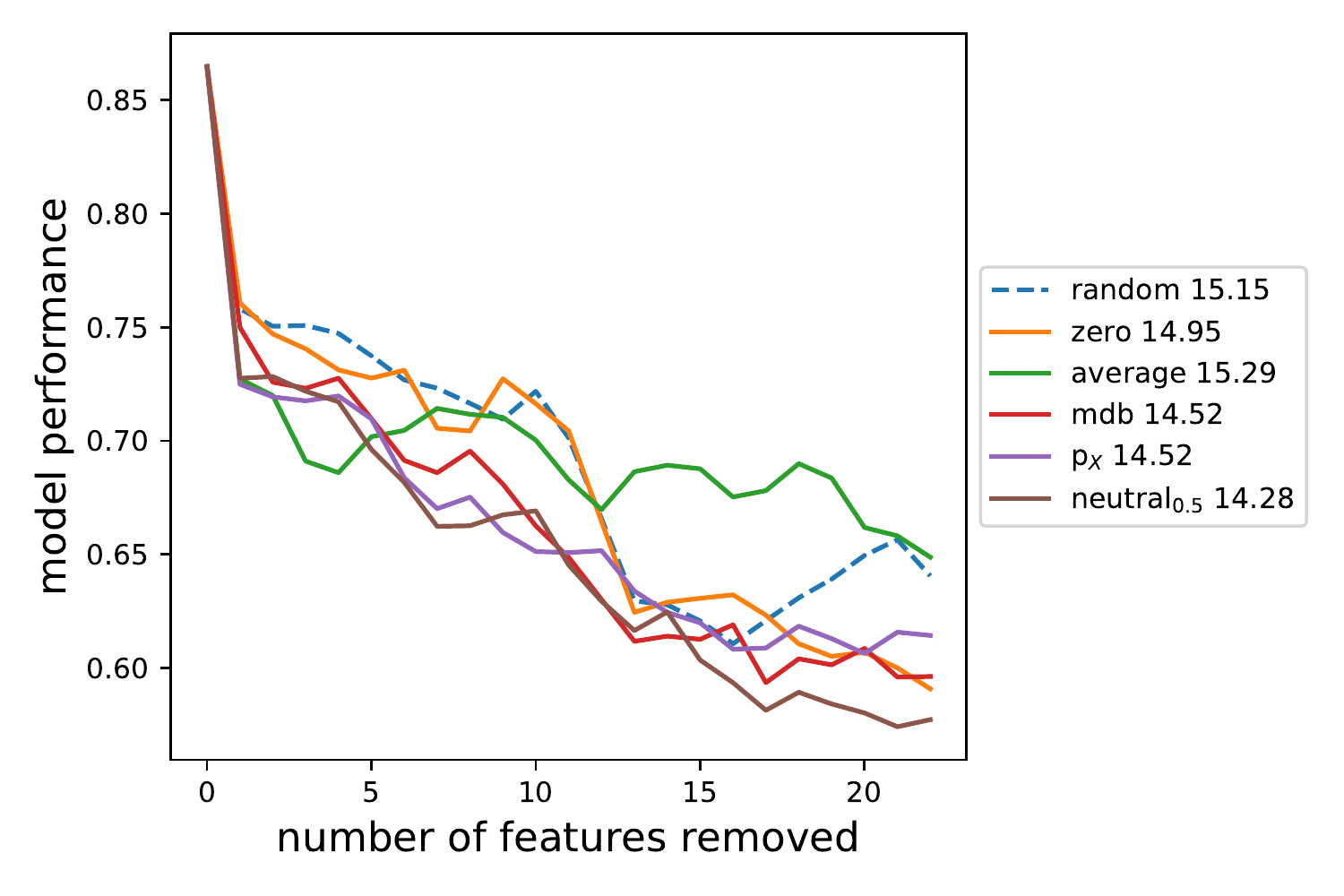}
		\caption{ROAR credit card.}
		\label{fig:roar_ccd}
	\end{subfigure}
	\caption{Perturbation tests on the credit card dataset. 
		}
\end{figure}

Fig.~\ref{fig:local_synth}, \ref{fig:local_ccd} show changes in expected absolute logits. The only baseline that in all datasets guarantees a monotonic decrease to zero in the information content when removing features is the \textbf{neutral}$_\alpha$. Thus, this is the only baseline that ensures convergence to a meaningful value when removing features.
Fig.~\ref{fig:roar_synth}, \ref{fig:roar_ccd} show ROAR scores. While in the synthetic dataset \textbf{p}$_{X}$ achieves the best score followed by \textbf{zero} and \textbf{neutral}$_\alpha$, in the real dataset, it is the \textbf{neutral}$_\alpha$ baseline that achieves the best score. Since \textbf{p}$_{X}$ and \textbf{neutral}$_\alpha$ show similar ROAR scores, the two approaches do equally well in ranking features in order of importance.

	\section{Conclusion}
	\label{sec:conclusions}
	In this work, we have investigated the identification of baselines for Shapley values based attribution methods and MLPs.
We have introduced the concept of neutrality and fair baselines. Their combination has allowed us to develop a neutral baseline that provides direct interpretation of the Shapley values, being them calculated in relation to the decision threshold of the model. This is in contrast to the baseline methods, where explanations are provided with respect to some arbitrary value with no direct relation to the model.
%
Nevertheless, the computational cost of searching the \textbf{neutral}$_\alpha$ baseline increases exponentially with respect to the number of hidden layers.
Further, we did not analyse how to apply such method to recurrent networks and how to extend it to regression problems which we leave to future work.
%
	
	
	\begin{footnotesize}

        \bibliographystyle{unsrtnat}
		\bibliography{ms}
		
	\end{footnotesize}
	
	
\end{document}